\documentclass{article} 
\usepackage[T1]{fontenc}
\usepackage{iclr2027_conference,times}

\usepackage{graphicx}
\usepackage{float}
\usepackage{placeins}
\usepackage{array}
\usepackage{booktabs}
\usepackage{xcolor}
\newfloat{algorithm}{t}{loa}
\floatname{algorithm}{Algorithm}

\usepackage{amsmath,amssymb,mathtools,bm}
\usepackage{amsthm}

\newtheorem{proposition}{Proposition}

\theoremstyle{remark}

\newcommand{\sg}{\operatorname{sg}}

\usepackage{amsmath,amsfonts,bm}

\def\eqref#1{equation~\ref{#1}}

\def\1{\bm{1}}

\DeclareMathAlphabet{\mathsfit}{\encodingdefault}{\sfdefault}{m}{sl}
\SetMathAlphabet{\mathsfit}{bold}{\encodingdefault}{\sfdefault}{bx}{n}

\usepackage{hyperref}
\hypersetup{
    pdftitle={From Gradients to Capabilities: Understanding Multi-Teacher On-Policy Distillation},
    pdfauthor={Siqi Zhu, Suozhi Huang, Kaixuan Zhang, Yuheng Yang, Zhanyang Jin, Yihang Sun, Jiaxuan You}}
\usepackage{url}

\title{From Gradients to Capabilities: Understanding Multi-Teacher On-Policy Distillation}

\author{%
\begin{minipage}[t]{\dimexpr\textwidth-2\tabcolsep\relax}
\centering
\makebox[\linewidth][c]{\fontsize{10}{14}\selectfont\bfseries
Siqi Zhu\textsuperscript{1}\hspace{1.5em}Suozhi Huang\textsuperscript{2}\hspace{1.5em}%
Kaixuan Zhang\textsuperscript{1}\hspace{1.5em}Yuheng Yang\textsuperscript{3}}\\[3pt]
\makebox[\linewidth][c]{\fontsize{10}{14}\selectfont\bfseries
Zhanyang Jin\textsuperscript{1}\hspace{1.5em}Yihang Sun\textsuperscript{1}\hspace{1.5em}%
Jiaxuan You\textsuperscript{1}}\\[0.7em]
\makebox[\linewidth][c]{\normalfont\small
\textsuperscript{1}University of Illinois Urbana-Champaign\quad
\textsuperscript{2}Princeton University\quad
\textsuperscript{3}Westlake University}
\end{minipage}
}

\iclrfinalcopy
\begin{document}

\maketitle
\fancyhead{}
\renewcommand{\headrulewidth}{0pt}

\begin{abstract}
Multi-teacher on-policy distillation (MOPD) aims to combine the strengths of RL-trained teachers in a single student, but how teacher signals affect parameter changes remains underexplored. We study Qwen3-1.7B with four domain teachers trained with RL from the same initialization as the student, comparing gradients, optimizer updates, and task learning curves, with additional SmolLM3-3B diagnostics. We find that several factors influence teacher signals. First, loss averaging implicitly weights responses: token averaging favors longer responses, and equalizing domain contributions retains this weighting within domains. 
Second, Adam's first moment reduces differences in parameter updates: the cosine similarity is 0.83 between teachers and 0.96 between averaging rules, despite differences in raw gradients.
Third, BF16 rounding hides small changes: about 97\% of FP32 master weights differ from initialization, but only 7--11\% of BF16 weights do. Finally, the top-64 intersection KL gradient closely matches Qwen's full-vocabulary gradient, but the effect on task performance depends on averaging: mathematics accuracy is 2.6 points higher than with sampled-token policy-gradient (PG) under response averaging and 2.1 points lower under global token averaging.
\end{abstract}

\section{Introduction}
\label{sec:intro}

Reinforcement learning produces strong language-model specialists in mathematics, coding, and instruction following, but combining their strengths remains challenging \citep{ma2026mopdmultiteacheronpolicydistillation}. Multi-teacher on-policy distillation (MOPD) trains a single student using next-token probabilities from domain teachers on student-generated responses \citep{xiaomi2026mimov2flash,nvidia2026nemotron3ultra,kimiteam2026kimik3openfrontier,deepseekai2026deepseekv4highlyefficientmilliontoken}. Yet strong teachers do not ensure that the student acquires every teacher's strengths \citep{gao2026openmopddiagnosingfixingcapability}. Because all teachers influence shared student parameters, we study how teacher losses combine and how the optimizer converts their gradients into parameter changes, connecting supervision to updates and task performance.

Existing MOPD recipes average losses within responses \citep{ma2026mopdmultiteacheronpolicydistillation}, normalize at the response and domain levels \citep{bai2026agentsa1}, or correct unequal domain token shares \citep{gao2026openmopddiagnosingfixingcapability}. These choices change the weights assigned to teachers and responses, but when does correcting unequal loss weights improve capability integration?
Prior work also compares losses computed from sampled tokens and top-$k$ vocabulary \citep{ma2026mopdmultiteacheronpolicydistillation}. However, Nemotron 3 Ultra reports no performance advantage from broader vocabulary supervision in its preliminary experiments \citep[Section~3.3.5]{nvidia2026nemotron3ultra}. When does using more vocabulary candidates improve task performance? Separately, parameter analyses show concentrated, sparse changes \citep{yu2026denseupdates,shen2026opdgeometry}. We therefore also explore how the optimizer and numerical precision affect observed parameter changes.

\begin{figure}[t]
\centering
\includegraphics[width=\linewidth]{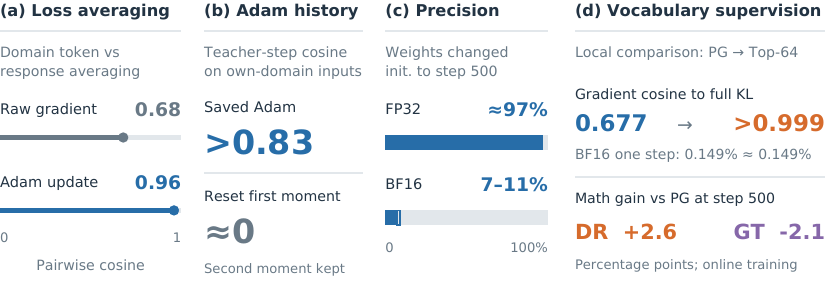}
\caption{\textbf{From supervision choices to parameter updates and task gains.}
(a) Averaging changes gradients more than Adam updates.
(b) Shared momentum aligns teacher updates.
(c) BF16 rounding hides widespread FP32 changes.
(d) Local gradient fidelity and one-step BF16 changes, followed by mathematics gains in separate 500-step training runs. DR/GT denote domain response/global token averaging.}
\label{fig:teaser}
\end{figure}

We use Qwen3-1.7B and RL teachers for mathematics, code, instruction following, and science, with additional diagnostics on SmolLM3-3B. As in existing MOPD recipes, these teachers are trained by RL from the student's initialization \citep{ma2026mopdmultiteacheronpolicydistillation,gao2026openmopddiagnosingfixingcapability}. To isolate loss choices, we compare gradients and optimizer steps from identical student parameters, responses, and optimizer states. We also track per-task learning curves as training progresses. Three findings show how supervision choices affect parameter updates and task performance (Figure~\ref{fig:teaser}):

\noindent\textbf{1. Loss averaging acts as implicit response weighting, and its effect depends on the task and objective.}
Global token averaging gives longer-response domains more loss weight. Equalizing domain weights removes this imbalance, but a covariance identity shows that it still shifts each domain's gradient toward longer responses; domain response averaging weights responses equally. No averaging rule performs best on every task, and the distillation loss changes how loss weights translate into task scores.

\noindent\textbf{2. Adam's first moment aligns updates, and BF16 rounding hides widespread changes.}
Updates computed from different teachers' losses point in similar directions even when their current gradients are nearly orthogonal. Using the PG loss, SGD achieves higher four-task mean scores than Adam with all loss averaging rules. After MOPD training, most student parameters differ from initialization in FP32, but only a small fraction differ after BF16 rounding.

\noindent\textbf{3. Better gradient approximation does not guarantee larger capability gains.}
Gradients from shared teacher and student top-64 vocabulary closely match Qwen's full-vocabulary direction. Final mathematics accuracy with the top-64 intersection KL loss is 2.6 points higher than with the PG loss under response averaging, but 2.1 points lower under global token averaging. Although the PG loss is unbiased for the full-vocabulary gradient, a single sampled token gives a noisy estimate, and neither unbiasedness nor closer approximation predicts which tasks improve. In SmolLM, top-64 vocabulary retain most student probability on average but still leave gradient error.


\section{Preliminaries}
\label{sec:setting}

\paragraph{Multi-teacher on-policy distillation (MOPD).}
For a prompt $x\sim\mathcal D_d$ from domain $d$, the student generates $y\sim p_\theta(\cdot\mid x)$. The domain's frozen teacher $q_d$ provides next-token probabilities at each prefix $h_t=(x,y_{<t})$. Batch $\mathcal B$ pools responses from different domains; each response $i$ uses its domain's teacher and averages $T_i$ valid token losses: $\ell_i=T_i^{-1}\sum_t\ell_{i,t}$. We minimize
\begin{equation}
\mathcal L(\theta)=\sum_{i\in\mathcal B}w_i\ell_i(\theta),
\qquad w_i\geq0,\quad\sum_{i\in\mathcal B}w_i=1.
\label{eq:mopd}
\end{equation}
Each training step uses fresh responses, whose text stays fixed while we differentiate the loss with respect to student parameters $\theta$.

\paragraph{Distillation losses.}
Write $p_\theta(v)=p_\theta(v\mid h_t)$ and $q_d(v)=q_d(v\mid h_t)$. The full-vocabulary KL loss computes reverse KL using all tokens in vocabulary $\mathcal V$; the sampled-token policy-gradient (PG) loss uses the token $y_t$ sampled by the student:
\begin{equation}
\label{eq:full}
\ell_{\mathrm{full}}=\sum_{v\in\mathcal V}p_\theta(v)\log\frac{p_\theta(v)}{q_d(v)},\qquad
\ell_{\mathrm{PG}}(y_t)=-\sg[\log q_d(y_t)-\log p_\theta(y_t)]\log p_\theta(y_t)
\end{equation}

Here $\sg$ treats its argument as constant when differentiating. The PG loss is a surrogate: its expected gradient under $y_t\sim p_\theta$ equals the full reverse-KL gradient (Appendix~\ref{app:pg_unbiasedness}); individual samples can differ.
Let $S_k$ be the student's $k$ most probable tokens under $p_\theta$ and $T_k$ the teacher's $k$ most probable tokens under $q_d$. Top-$k$ OPD usually restrict the KL to one of these sets (student-top-$k$ on $S_k$ or teacher-top-$k$ on $T_k$). Our top-$k$ intersection KL loss instead uses only the tokens in both sets, $I=I_k=S_k\cap T_k$ (Appendix~\ref{app:intersection_topk}):
\begin{equation}
\begin{aligned}
\bar p_\theta^I(v)&=\frac{p_\theta(v)}{\sum_{u\in I}p_\theta(u)},\qquad
\bar q_d^I(v)=\frac{q_d(v)}{\sum_{u\in I}q_d(u)},\quad v\in I,\\
\ell_{\mathrm{top}\text{-}k}&=\sum_{v\in I}\bar p_\theta^I(v)
\log\frac{\bar p_\theta^I(v)}{\bar q_d^I(v)}.
\end{aligned}
\label{eq:topk}
\end{equation}
Full-vocabulary softmax probabilities are rescaled to sum to one within $I_k$, comparing candidates' relative probabilities; empty intersections give zero loss. Because the renormalized probabilities depend only on logits in $I_k$, the loss has zero gradient with respect to the logits of tokens outside $I_k$. The \emph{supervision density} counts vocabulary candidates: 1 for PG, $|I_k|\leq k$ for Top-$k$, and $|\mathcal V|$ for full-vocabulary KL.

\paragraph{Domain weighting and loss normalization.}
Let $\mathcal B_d$ contain domain-$d$ responses and let $\lambda_d\geq0$, $\sum_d\lambda_d=1$. For nonempty responses and domains, response $i\in\mathcal B_d$ has weights
\begin{equation}
w_i^{\mathrm{GT}}=\frac{T_i}{\sum_{j\in\mathcal B}T_j},\qquad
w_i^{\mathrm{DT}}=\frac{\lambda_d T_i}{\sum_{j\in\mathcal B_d}T_j},\qquad
w_i^{\mathrm{DR}}=\frac{\lambda_d}{|\mathcal B_d|}.
\label{eq:averaging}
\end{equation}
\emph{Global token averaging} (GT) weights all response tokens equally, so each domain's weight follows its token share. \emph{Domain token averaging} (DT) fixes each domain's weight at $\lambda_d$ and weights responses by length. \emph{Domain response averaging} (DR) keeps these domain weights but weights responses equally within each domain. Equal prompt counts can still give unequal domain weights under GT.

\section{Loss Averaging Implicitly Weights Domains and Responses}
\label{sec:normalization}

Open-MOPD reports a large instruction-following deficit relative to single-domain students, with math/code responses averaging about 10,500 tokens and IF responses 409 tokens \citep{gao2026openmopddiagnosingfixingcapability}. In our less extreme length regime, we compare the three averaging rules (Section~\ref{sec:setting}) to separate weighting between and within domains.

\begin{table}[t]
\centering
\small
\setlength{\tabcolsep}{4pt}
\renewcommand{\arraystretch}{1.08}
\newcommand{\tablescore}[2]{#1\,{\scriptsize$\pm\,#2$}}
\newcommand{\tableinitial}[1]{#1\,{\scriptsize$\phantom{\pm\,0.00}$}}
\caption{\textbf{MOPD scores by loss averaging and distillation loss.}
Scores (\%) are mean $\pm$ standard deviation across five training seeds; Mean equally weights the four task means. Domain responses and Domain tokens assign equal domain weights, averaging responses equally or by length, respectively. Global token averaging weights all response tokens equally.}
\label{tab:main_outcomes}
\begin{tabular*}{\linewidth}{@{\extracolsep{\fill}}lcccc@{\hspace{14pt}}c@{}}
\toprule
\textbf{Loss averaging} & \textbf{Math} & \textbf{Code} & \textbf{IF} & \textbf{Science} & \textbf{Mean} \\
\midrule
Initial student
& \tableinitial{68.80}
& \tableinitial{16.40}
& \tableinitial{18.33}
& \tableinitial{32.20}
& \tableinitial{33.93} \\
\midrule
\multicolumn{6}{@{}l}{\textbf{Adam, PG loss}} \\

Domain responses
& \tablescore{73.36}{0.79}
& \tablescore{18.26}{0.71}
& \tablescore{28.07}{1.19}
& \tablescore{35.94}{1.42}
& \tablescore{38.91}{0.69} \\

Domain tokens
& \tablescore{73.28}{0.69}
& \tablescore{19.24}{0.83}
& \tablescore{27.07}{1.01}
& \tablescore{34.93}{1.35}
& \tablescore{38.63}{0.65} \\

Global tokens
& \tablescore{76.32}{0.73}
& \tablescore{17.91}{0.66}
& \tablescore{27.13}{1.24}
& \tablescore{33.81}{1.54}
& \tablescore{38.79}{0.70} \\

\midrule
\multicolumn{6}{@{}l}{\textbf{Adam, top-16 intersection KL loss}} \\

Domain responses
& \tablescore{75.48}{0.89}
& \tablescore{17.73}{0.62}
& \tablescore{26.80}{1.12}
& \tablescore{35.98}{1.31}
& \tablescore{39.00}{0.67} \\

\midrule
\multicolumn{6}{@{}l}{\textbf{Adam, top-64 intersection KL loss}} \\

Domain responses
& \tablescore{75.96}{0.80}
& \tablescore{18.29}{0.75}
& \tablescore{27.33}{1.00}
& \tablescore{34.68}{1.48}
& \tablescore{39.07}{0.68} \\

Domain tokens
& \tablescore{73.80}{0.97}
& \tablescore{18.86}{0.69}
& \tablescore{28.33}{1.35}
& \tablescore{35.19}{1.27}
& \tablescore{39.05}{0.72} \\

Global tokens
& \tablescore{74.20}{0.80}
& \tablescore{19.77}{0.88}
& \tablescore{28.67}{1.20}
& \tablescore{33.38}{1.61}
& \tablescore{39.01}{0.59} \\

\midrule
\multicolumn{6}{@{}l}{\textbf{SGD, PG loss}} \\

Domain responses
& \tablescore{76.48}{0.73}
& \tablescore{19.53}{0.79}
& \tablescore{28.20}{1.12}
& \tablescore{35.56}{1.36}
& \tablescore{39.94}{0.52} \\

Domain tokens
& \tablescore{75.20}{0.87}
& \tablescore{20.28}{0.92}
& \tablescore{28.47}{1.12}
& \tablescore{33.88}{1.45}
& \tablescore{39.46}{0.55} \\

Global tokens
& \tablescore{74.96}{1.05}
& \tablescore{19.24}{0.73}
& \tablescore{27.53}{1.35}
& \tablescore{35.32}{1.52}
& \tablescore{39.26}{0.60} \\

\bottomrule
\end{tabular*}
\end{table}

\subsection{Loss averaging changes gradients and per-task scores, but barely the mean}

With equal prompt counts, mathematics supplies about three times as many tokens as instruction following, giving it three times the loss weight under global token averaging (GT; Figure~\ref{fig:normalization_capability}a). Domain token averaging (DT) fixes domain weights but still weights longer responses more. Domain response averaging (DR) also weights responses equally.
For response gradients $g_i=\nabla_\theta\ell_i$ and unweighted domain means $\bar g_d,\bar T_d$, DT's within-domain gradient is
\begin{equation}
\frac{\sum_{i\in\mathcal B_d}T_i g_i}{\sum_{i\in\mathcal B_d}T_i}
=\bar g_d+\frac{\operatorname{Cov}_d(T,g)}{\bar T_d}.
\label{eq:weighting}
\end{equation}
The covariance term measures how length weighting shifts the domain gradient when longer responses have systematically different gradients. DR removes this term. Open-MOPD's token-share correction recovers DT \citep{gao2026openmopddiagnosingfixingcapability}: it fixes domain weights but retains length weighting within each domain. Appendix~\ref{app:normalization} extends this identity to arbitrary response and position weights.

To measure the change in gradient direction, we compare DT and DR on the same responses at fixed student parameters. Their mean gradient cosine is 0.68 in six batches (Figure~\ref{fig:normalization_capability}b). In full training runs, the range of four-task mean scores among GT, DT, and DR after 500 steps is 0.28 percentage points under Adam with the PG loss and 0.06 with the top-64 intersection KL loss. Under SGD with the PG loss, this range is 0.68 points (Table~\ref{tab:main_outcomes}).

GT assigns mathematics about 44\% of the loss weight (Figure~\ref{fig:normalization_capability}a). Relative to DR, GT raises mathematics accuracy by 3.0 points and lowers science by 2.1 points with the PG loss. With top-64, mathematics falls by 1.8 points despite its larger weight, while code rises by 1.5 points despite its smaller weight (Table~\ref{tab:main_outcomes}). Thus, the loss changes how weighting affects task performance (Section~\ref{sec:density_tasks}).

\subsection{Adam reduces differences in update direction}

For each of these six batches, we apply the DT and DR gradients separately using the same saved Adam state. The resulting one-step parameter updates have a mean cosine of 0.96 (Figure~\ref{fig:normalization_capability}b). Adam mixes current gradients with past gradients through its first moment and rescales coordinates through its second moment \citep{kingma2015adam}.

\begin{figure}[t]
\centering
\includegraphics[width=\linewidth]{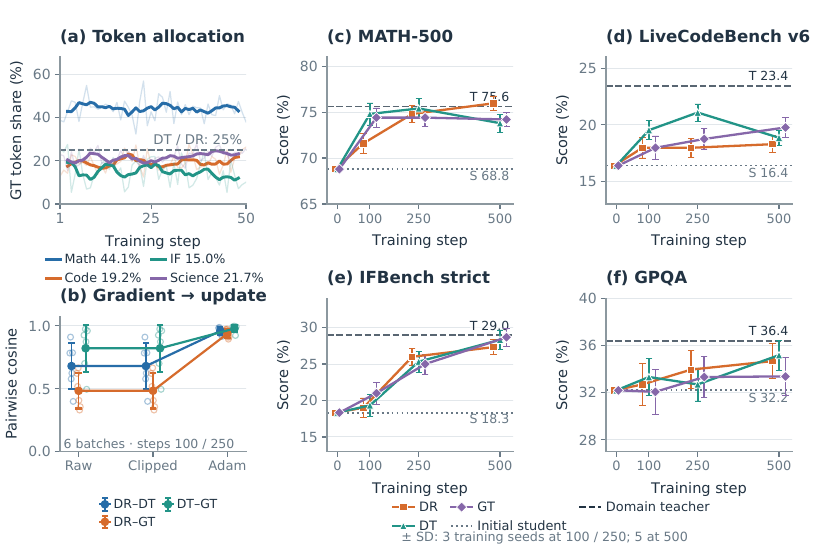}
\caption{\textbf{Loss averaging changes gradients, while shared Adam state aligns updates.}
(a) Domain token shares under global token averaging.
(b) Gradient and update cosines for six batches (three per checkpoint).
(c--f) Per-task learning curves with the top-64 intersection KL loss. Error bars show standard deviations from three training seeds at steps 100 and 250, and five at step 500.
GT/DT/DR denote global token, domain token, and domain response averaging. S/T mark the initial student and domain teacher.}
\label{fig:normalization_capability}
\end{figure}

DR triggers more clipping, which rescales gradients without changing direction. Resetting Adam's first moment shrinks steps; resetting both moments enlarges them (Figure~\ref{fig:normalization_interventions}a,b). For each rule, we scale a single SGD step to match its saved-state Adam step norm. The held-out KL change differs more between rules under SGD than under saved-state Adam (panel c). During training, Adam step norms become similar despite distinct gradient norms, but cumulative parameter changes remain different (panels d--f).
\begin{figure}[t]
\centering
\includegraphics[width=\linewidth]{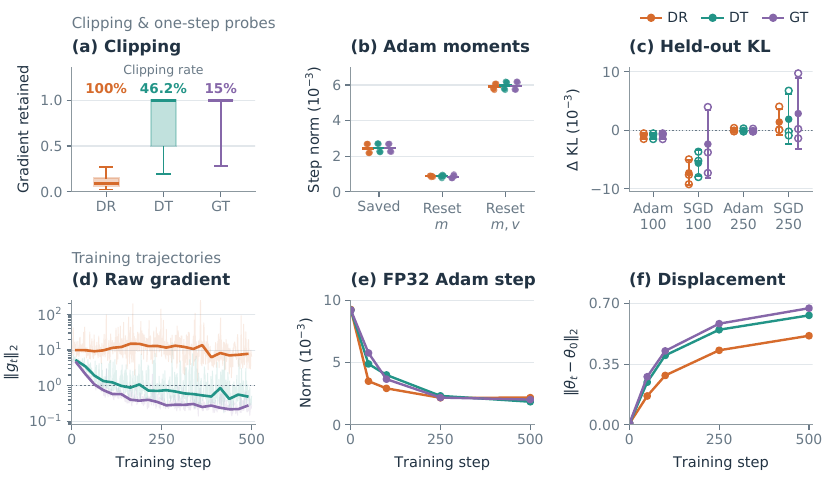}
\caption{\textbf{Optimizer processing attenuates differences between averaging rules.}
(a) Gradient retention and clipping rates.
(b) Adam step norms with saved or reset moments (step count retained).
(c) Held-out reverse KL after minus before one step; SGD matches each rule's saved-Adam step norm. Open points: batches; filled points/bars: mean/SD.
(d--f) Raw gradient norms, FP32 step norms, and cumulative FP32 displacement under Adam with the top-64 intersection KL loss. Faint/solid traces show individual steps/25-step medians; dotted line: clipping threshold. GT/DT/DR denote global token, domain token, and domain response averaging.}
\label{fig:normalization_interventions}
\end{figure}

\subsection{Raising the response cap shifts token-averaged objectives most}

We compare maximum training response lengths of 4,096 (4K) and 8,192 (8K) generated tokens per response. Unlike RL with final-answer rewards \citep{yu2025dapo}, on-policy distillation can supervise incomplete responses using teacher probabilities at each prefix \citep{agarwal2024gkd}. Increasing this limit can add supervised token positions, redistributing even DR's fixed response weight.
At three checkpoints, paired batches generated under the 4K and 8K limits share student and Adam states. Relative raw-gradient differences, $\|g_{8K}-g_{4K}\|/\|g_{4K}\|$, follow GT $>$ DT $>$ DR (Figure~\ref{fig:response_cap}b); clipping and Adam can alter this ranking. In training, the 8K-minus-4K mean-score differences at step 100 are GT $+1.05$, DT $+0.32$, and DR $-0.83$ percentage points (panel c).

\begin{figure}[!ht]
\centering
\includegraphics[width=\linewidth]{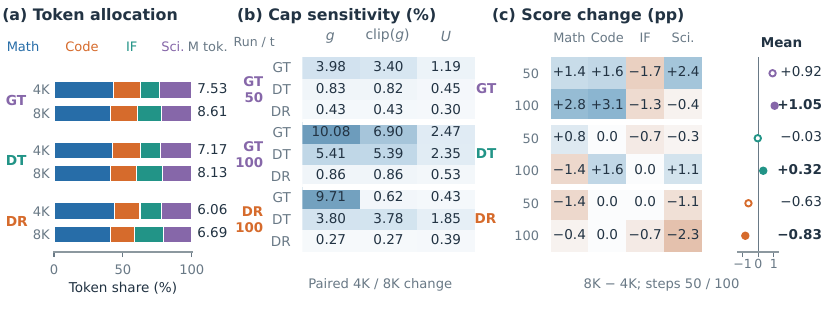}
\caption{\textbf{Loss averaging influences the effect of doubling the maximum training response length from 4,096 (4K) to 8,192 (8K) tokens.}
(a) Domain token shares and total response tokens (millions) during 100 steps.
(b) Relative 4K/8K changes in gradients ($g$), clipped gradients, and Adam steps ($U$); groups name source 4K run and training step, rows mark averaging rules. Each group averages three paired batches.
(c) Task score differences (8K minus 4K, percentage points); adjacent points show equal-domain means at steps 50 (open) and 100 (filled). GT/DT/DR denote global token, domain token, and domain response averaging.}
\label{fig:response_cap}
\end{figure}

\section{BF16 Rounding and Adam Momentum Obscure How Teachers Change the Student}
\label{sec:subnetworks}

\subsection{BF16 rounding hides widespread FP32 changes}

Under domain response averaging, about 97\% of the MOPD student's FP32 master weights differ from initialization, versus 7--9\% after rounding them to BF16 (7--11\% including DT and GT; Figure~\ref{fig:subnetwork_overview}a). The master weights retain higher precision for optimization.
For each parameter, change means its final value minus its initial value. We rank parameters by squared change. About one third account for 90\% of the total in FP32, versus about 4\% in BF16. Rounding therefore hides small changes and makes widespread parameter movement appear sparse.

\begin{figure}[t]
\centering
\includegraphics[width=\linewidth]{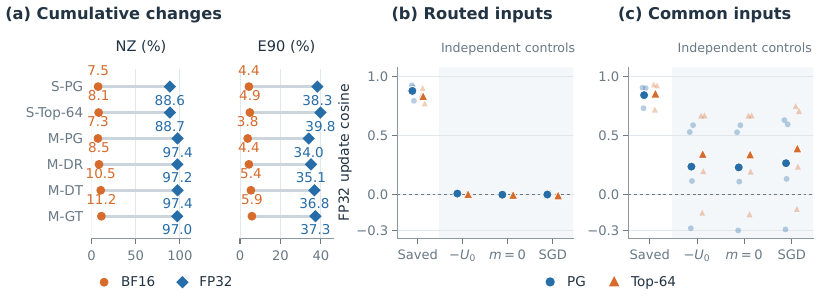}
\caption{\textbf{BF16 rounding hides small changes; shared Adam momentum aligns student updates.}
(a) FP32/BF16 parameter fractions: changed from initialization (NZ), or carrying 90\% of squared change (E90). S/M: math-only/MOPD training; PG: sampled-token loss; Top-64: intersection KL loss. M-DR/M-DT/M-GT: Top-64 with response/domain-token/global-token averaging; other rows average responses.
(b,c) Pairwise update cosines: own-domain responses (Routed) or shared prefixes (Common). Saved retains trained Adam state; $-U_0$ subtracts the zero-gradient step; $m=0$ resets the first moment. Small points average teacher pairs within batches; large points average batches.}
\label{fig:subnetwork_overview}
\end{figure}

\subsection{Momentum hides differences between teachers}

We fix parameters and Adam state from the MOPD student trained with the PG loss. For each teacher, we compute the corresponding student loss gradient and simulate one FP32 parameter update. Adam's first moment mixes past gradients into each update. When teachers score student responses from their own domains, the mean cosine, computed using all teacher pairs and batches, exceeds 0.83 with the PG and top-64 intersection KL losses. Resetting the first moment while retaining the second, or using norm-matched SGD, reduces it to near zero (Figure~\ref{fig:subnetwork_overview}b). Shared momentum therefore explains most of this alignment.
Momentum can move parameters even with zero current gradient. We subtract this baseline to measure how the current gradient changes Adam's update, including its second-moment scaling. Let $g_d$ be the student loss gradient from teacher $d$'s supervision, $H$ the saved Adam moments and step counter, and $U(g;H)$ the resulting parameter update:
\begin{equation}
\delta U_d=U(g_d;H)-U(0;H).
\label{eq:history_subtracted}
\end{equation}
The mean pairwise cosine of these differences is below 0.01 on domain-specific inputs. Having all teachers score identical student prefixes holds inputs fixed and raises it to about 0.24--0.35 (Figure~\ref{fig:subnetwork_overview}c). Teacher and input differences can therefore remain hidden within aligned Adam updates.

\subsection{Momentum-free SGD achieves higher mean scores}

In MOPD, the student generates new responses as it learns, changing the prefixes teachers supervise. Momentum may retain directions favored by earlier responses and delay adaptation to current supervision, a concern also raised for reinforcement learning \citep{mukherjee2026sgd}. In full training runs with the PG loss, momentum-free SGD, which keeps no optimizer history, achieves higher four-task mean scores than Adam on all three averaging rules (DR: 39.94 vs.\ 38.91; DT: 39.46 vs.\ 38.63; GT: 39.26 vs.\ 38.79; Table~\ref{tab:main_outcomes}).

\section{Closer Vocabulary Gradients Do Not Consistently Improve Capabilities}
\label{sec:density}

\subsection{Top-64 intersections nearly recover the full-vocabulary gradient}
\label{sec:coverage_gradient}

With fixed domain weights, vocabulary truncation can change the student's combined loss gradient \citep{anshumann2025sparse,huang2026tailaware}. We measure retained student probability, $c_k(h)=\sum_{v\in I_k}p_\theta(v\mid h)$, and agreement with the full-vocabulary KL gradient.
In Qwen, the top-64 intersection retains, on average, more than 99.9\% of total student probability and more than 99.97\% of the student's own top-64 mass. Its gradient direction nearly matches the full-vocabulary KL gradient (Figure~\ref{fig:supervision_overview}b). Averaging gradients from 1, 16, or 64 samples per fixed prefix raises cosine to this reference from 0.54 to 0.87 and 0.97 (Figure~\ref{fig:sampling_history_controls}a).

\begin{figure}[t]
\centering
\includegraphics[width=\linewidth]{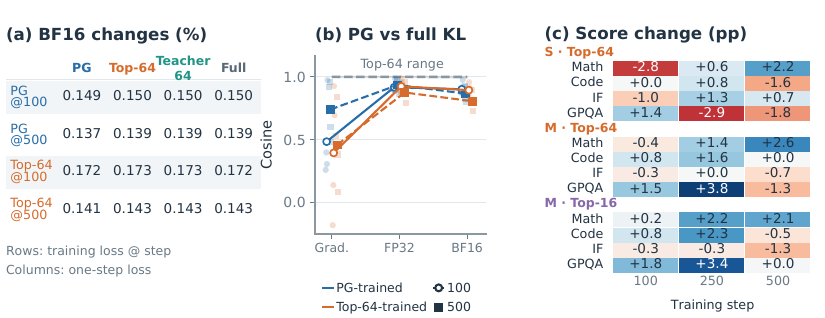}
\caption{\textbf{Top-64 closely matches full-vocabulary gradient directions in Qwen.}
(a) Rows: training loss@step. Columns compare one-step BF16 changed fractions for different losses at fixed parameters, responses, and Adam state.
(b) PG-to-full gradient and update cosines across training snapshots; gray shows Top-64-to-full ranges.
(c) Score differences from PG under response averaging. PG: sampled-token loss; Top-16/Top-64: intersection KL losses; Teacher-64: teacher-supported loss without renormalization (Appendix~\ref{app:losses}); Full: full-vocabulary KL loss. S/M: mathematics-only/MOPD training.}
\label{fig:supervision_overview}
\end{figure}

SmolLM3-3B illustrates the limits of mean coverage: top-64 retains 97.42\% initially and 98.56\% after training, yet 8.42\% and 3.65\% of weighted prefixes, respectively, retain less than 90\% of student probability. After training, relative gradient error spans 2.16--24.54\% among diagnostic batches (Figure~\ref{fig:smollm_supervision}).

\begin{figure}[t]
\centering
\includegraphics[width=\linewidth]{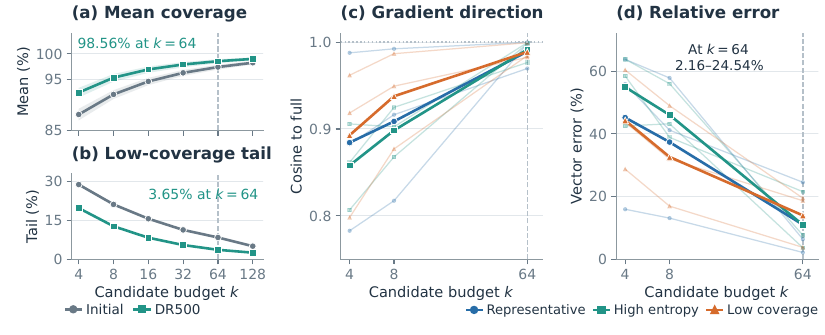}
\caption{\textbf{High mean coverage can hide gradient error in SmolLM.}
(a,b) Mean intersection coverage and fraction below 90\%; domains, responses within domains, and positions within responses are equally weighted. Bands: response-cluster 95\% intervals.
(c,d) Gradient cosine and relative error against full-vocabulary KL after training; relative error is the gradient-difference norm divided by the full-gradient norm. Dashed lines: $k=64$. Thin/bold lines: batches/means. Initial: initialization; DR500: 500 steps with top-64 intersection KL and response averaging.}
\label{fig:smollm_supervision}
\end{figure}

\begin{figure}[t]
\centering
\includegraphics[width=0.9\linewidth]{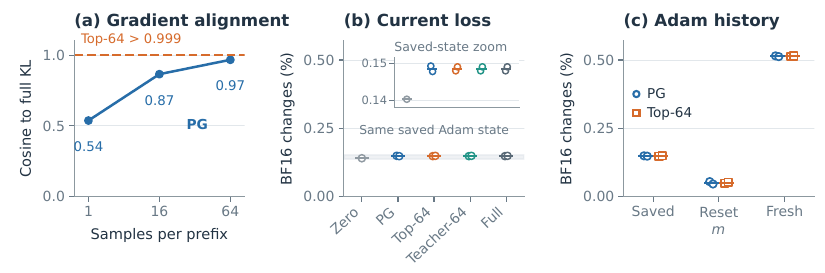}
\caption{\textbf{Sampling improves gradient alignment; Adam history affects BF16 changes.}
(a) Mean gradient cosine to full KL.
(b,c) Changed BF16 fractions with shared axes; inset repeats (b)'s categories.
PG: sampled-token loss; Top-64: intersection KL; Teacher-64: teacher-supported loss; Full: full-vocabulary KL.
Saved retains Adam state; reset-$m$ clears its first moment; fresh Adam clears both moments and step count.
In (b,c), open points show measurements; bars show means.}
\label{fig:sampling_history_controls}
\end{figure}

\subsection{Supervision breadth barely affects the fraction of changed BF16 weights}

Starting from identical student parameters and Adam state, the PG, top-64 intersection KL, teacher-top-64, and full-vocabulary KL losses each change nearly the same fraction of BF16 weights in one step (Figure~\ref{fig:supervision_overview}a).
Adam history strongly affects these counts: even with zero current gradient, about 0.14\% of BF16 weights change, close to the supervised fractions. Resetting only the first moment lowers the fraction; fully resetting Adam raises it (Figure~\ref{fig:sampling_history_controls}b,c). Therefore, the fraction of BF16 weights that change reflects optimizer history and rounding (Section~\ref{sec:subnetworks}). Coverage and gradient agreement assess the truncation approximation.

\subsection{Task gains from top-\texorpdfstring{$k$}{k} depend on the averaging rule}
\label{sec:density_tasks}

Under response averaging, top-64 exceeds PG in mathematics accuracy by 2.6 percentage points, while instruction following and science scores decrease; the four-task mean rises by 0.16 points. Under global token averaging, mathematics accuracy falls by 2.1 points, while code and instruction following improve; the mean rises by 0.21 points (Table~\ref{tab:main_outcomes}). Close gradient agreement therefore coexists with task-specific gains and losses, and response weighting changes which tasks benefit.

These outcomes show that gradient fidelity at fixed prefixes is not a sufficient predictor of task-level gains: PG is unbiased for the full reverse-KL gradient, and top-64 approximates its direction with cosine above 0.999 in Qwen, yet their task effects differ across averaging rules.

\section{Related Work}

\paragraph{Distillation objectives and vocabulary truncation.}
MiniLLM uses reverse KL \citep{gu2024minillm}; GKD distills student-generated responses with alternative divergences \citep{agarwal2024gkd}; $f$-DISTILL derives token losses from sequence $f$-divergences \citep{wen2023fdivergence}. Sparse Logit Sampling examines top-$k$ bias \citep{anshumann2025sparse}; Tail-Aware OPD restores discarded tail-mass supervision \citep{huang2026tailaware}. We compare coverage, gradient fidelity, and task gains, accounting for sampling noise and optimizer processing.

\paragraph{Multi-teacher distillation and loss weighting.}
MOPD combines specialized teachers \citep{ma2026mopdmultiteacheronpolicydistillation}. Open-MOPD corrects domain token shares but retains within-domain length weighting \citep{gao2026openmopddiagnosingfixingcapability}. GradNorm uses gradient norms to balance task training rates \citep{chen2018gradnorm}. DAPO averages token losses \citep{yu2025dapo}, while Dr.~GRPO analyzes length and reward-normalization bias \citep{liu2025understandingr1zero}. We derive response weights' effects through their covariance with gradients and measure changes in Adam updates and task scores.

\paragraph{Parameter changes and gradient alignment.}
Task arithmetic combines parameter differences from a shared initialization \citep{ilharco2023taskarithmetic}. Prior work studies small parameter subsets supporting RL \citep{mukherjee2025subnetworks}, concentrated changes among distilled students \citep{yu2026denseupdates}, and low-dimensional OPD updates \citep{shen2026opdgeometry}. PCGrad projects conflicting gradients \citep{yu2020gradientsurgerymultitasklearning}; CAGrad uses worst-task improvement \citep{liu2024conflictaversegradientdescentmultitask}. We separate current supervision from Adam history, find nearly orthogonal teacher increments on routed inputs, and distinguish FP32 movement from BF16 sparsity.

\section{Discussion}

We study how multi-teacher supervision becomes capability through student gradients and parameter updates. Domain balancing retains within-domain length weighting; Adam history aligns updates; BF16 rounding hides small FP32 changes. Closer full-vocabulary gradient approximation does not consistently improve task scores. These findings support choosing response weighting, vocabulary supervision, and optimizer settings together for the target tasks and training budget. 
Our findings are derived from limited model families and scales, and the observation may differ in larger-scale settings.
Future work may explore whether smaller Adam $\beta_1$ improves adaptation to teacher signals, and whether gradient clipping and Adam's second-moment normalization could explain why losses with similar expected gradients produce different task outcomes.

\bibliography{iclr2027_conference}
\bibliographystyle{iclr2027_conference}

\clearpage
\appendix

\section{Experimental Setup}
\label{app:setup}

\subsection{Models, data, and training}

\textbf{MOPD training.} Our main experiments compare how normalization, loss formulation, and optimizer choice affect learning from the same four domain teachers. The student is Qwen3-1.7B \citep{qwen2025qwen3}, and the teachers were trained by reinforcement learning from the same initialization to specialize in mathematics, code, instruction following, and science. The main training comparisons use seeds 42, 43, 44, 45, 46.

AdamW uses learning rate $2.5\times10^{-7}$, $(\beta_1,\beta_2)=(0.9,0.98)$, $\epsilon=10^{-8}$, zero weight decay, a constant learning rate, and global norm-1 clipping. SGD uses learning rate $2.5\times10^{-2}$, following the setting selected by \citet{mukherjee2026sgd} through an analysis of AdamW's effective per-parameter learning rates and an SGD learning-rate sweep. Student responses are sampled at temperature 1 and top-$p=1$, with a 2,048-token prompt cap and a maximum response length of 4,096 tokens (4K), increased to 8,192 tokens (8K) in the response-length comparison (Figure~\ref{fig:response_cap}). Each training step aggregates one fresh response for each of 64 prompts.

MOPD training allocates 16 prompts to each of the four domains and assigns target domain weights of $0.25$ under the domain-balanced rules. Single-teacher training uses all 64 prompts from the teacher's domain. Top-64 and Top-16 denote intersection KL losses using Equation~\ref{eq:topk} with $k=64$ and $k=16$, respectively.

Experiments are conducted using 10 * NVIDIA RTX PRO 6000 Blackwell GPUs.

\paragraph{Domain data.}
We construct the four domain datasets from
\href{https://huggingface.co/datasets/nvidia/Nemotron-3-Nano-RL-Training-Blend}{Nemotron-3-Nano-RL-Training-Blend}. Table~\ref{tab:domain_data_sources}
lists the sources and the sizes of each dataset.
The mathematics pool combines DAPO-Math-17k and the mathematics
split of Skywork-OR1-RL-Data with OmniMath excluded. The code,
instruction-following, and science pools use the competitive-coding,
instruction-following, and STEM multiple-choice components,
respectively.

\begin{table}[h]
\centering
\small
\caption{Domain data sources and converted pool sizes before
stage-specific filtering.}
\label{tab:domain_data_sources}
\begin{tabular}{@{}lp{0.55\linewidth}r@{}}
\toprule
Domain & Upstream dataset & Rows \\
\midrule
Mathematics &
\path{BytedTsinghua-SIA/DAPO-Math-17k} and
\path{Skywork/Skywork-OR1-RL-Data} &
22,056 \\
Code &
\path{nvidia/Nemotron-RL-coding-competitive_coding} &
19,169 \\
Instruction following &
\path{nvidia/Nemotron-RL-instruction_following} &
16,575 \\
Science &
\path{nvidia/Nemotron-RL-knowledge-mcqa}
(STEM multiple-choice component) &
19,670 \\
\bottomrule
\end{tabular}
\end{table}

\paragraph{Teacher construction.}
We train four domain-specific Qwen3-1.7B teachers independently
from the common initialization using GRPO
\citep{shao2024deepseekmathpushinglimitsmathematical},
implemented in Slime codebase.
Each rollout batch contains 16 prompts, with the number of
responses per prompt specified in
Table~\ref{tab:grpo_teacher_training}.
Generation uses temperature 1, top-$p=1$, no top-$k$ sampling truncation,
a 2,048-token prompt limit, and an 8,192-token response limit.
We apply the chat template with thinking disabled. Training uses AdamW with a constant learning rate of $10^{-6}$,
$(\beta_1,\beta_2)=(0.9,0.9987381276)$,
$\epsilon=10^{-8}$, zero weight decay, no learning-rate warmup,
and global gradient-norm clipping at 1.

\begin{table}[h]
\centering
\small
\caption{Teacher GRPO setup.}
\label{tab:grpo_teacher_training}
\begin{tabular}{@{}lrrr@{}}
\toprule
Domain & Responses per prompt & Batch Size & Updates \\
\midrule
Mathematics           & 8  & 16 & 500 \\
Code                  & 16 & 16 & 500 \\
Instruction following & 16 & 16 & 500 \\
Science               & 4  & 16  & 1000 \\
\bottomrule
\end{tabular}
\end{table}

\subsection{Evaluation and scoring}

We evaluate mathematics on MATH-500 (500 questions) \citep{hendrycks2021math}, code on LiveCodeBench v6 (1,055 questions) \citep{jain2024livecodebench}, instruction following on IFBench strict (300 questions) \citep{pyatkin2025ifbench}, and science on GPQA Diamond (198 questions) \citep{rein2023gpqa}. The first three benchmarks use one greedy response per question. GPQA uses 32 responses per question at temperature 1 and top-$p=1$; average@32 accuracy first averages response correctness within each question and then averages these question scores. Evaluation permits up to 32,768 response tokens within a 32,768-token context. Reported cross-domain means weight the four benchmarks equally.

\section{Loss Averaging: Derivation and Controls}
\label{app:normalization}

Response length changes both a domain's contribution to the loss and the relative weights of responses within it. For a fixed rollout batch, token masks, and domain weights, define the empirical covariance with normalization $1/n_d$, where $n_d=|\mathcal B_d|$. Expanding gives $\operatorname{Cov}_d(T,g)=n_d^{-1}\sum_i T_i g_i-\bar T_d\bar g_d$; dividing by $\bar T_d$ gives Equation~\ref{eq:weighting}. Reweighting each domain's global-token contribution by its target weight divided by its observed token share gives domain-token averaging exactly. The remaining length weighting within each domain is captured by this covariance.

The same identity holds when $T_i$ is replaced by any nonnegative response weight $f_i$ with positive mean, and analogously for token-level weights: weighting by $f$ shifts the unweighted domain gradient by $\operatorname{Cov}_d(f, g)/\bar f_d$. Length is one implicit choice of $f$. UC-MOPD chooses weights explicitly by retaining rollouts with strong positive learning signal \citep{ucmopd}. Each choice changes the domain gradient through the covariance between $f$ and per-response gradients. Within each domain, averaging rules also reweight positions: late positions occur only in long responses, and DR gives each token of a long response a smaller weight, so DR moves loss weight toward early positions relative to DT and GT.

\subsection{Fixed-batch gradient and optimizer comparisons}
\label{app:fixed_batch_normalization}

To measure the effect of averaging alone, we reuse the same responses, token masks, teacher scores, and student state for all three rules. At steps 100 and 250 of the MOPD student trained with the PG loss, we cache three response batches (seeds 1042, 1043, and 1044), each with eight responses per domain, using the training prompt format and a 4,096-token response cap. We compute each response's mean top-64 gradient at the saved BF16 parameter values in FP32, then aggregate under domain-response (DR), domain-token (DT), and global-token (GT) averaging.

In these six batches, the DR--DT raw-gradient cosine ranges from 0.40 to 0.91 (mean 0.68; Figure~\ref{fig:normalization_capability}b). Therefore length weighting changes the gradient direction even when domain weights are fixed. We verify the covariance identity in FP64 for all 24 batch-domain groups; the largest relative residual is $1.53\times10^{-15}$.

The optimizer interventions in Figure~\ref{fig:normalization_interventions} ask how much of this gradient difference survives the update. Each aggregate gradient is clipped at norm 1 and applied to the same saved FP32 Adam state. The saved-state branch retains both moments; reset-$m$ zeros the first moment; reset-$(m,v)$ zeros both moments. Both reset branches retain the step counter. Momentum-free SGD follows the negative gradient and is rescaled separately for each averaging rule to match that rule's saved-state Adam step norm. Each branch applies one step to a checkpoint copy, casts the updated parameters to BF16, and measures full-vocabulary reverse KL on one held-out response per domain.

\section{Parameter and Optimizer Measurements}
\label{app:records}

\subsection{Recording gradients, optimizer steps, and weight changes}

To distinguish the current learning signal from the change stored in the model, we record gradients before and after clipping, FP32 optimizer steps, and cumulative displacement from initialization in both FP32 master weights and BF16 model weights. Global statistics count unique model coordinates, excluding vocabulary padding and duplicate entries for tied weights. Energy concentration ranks coordinates by absolute change and measures their share of squared Euclidean norm.

Checkpoint measurements preserve FP32 master weights, BF16 model weights, Adam moments, hyperparameters, and step counters. Student and teacher values are checked against their saved models in the same precision. Forward and backward passes use FP32 with TF32 disabled. Diagnostic steps use a non-fused Adam implementation applied to copies of the saved state; online training uses fused Adam.

\subsection{Separating the current gradient from optimizer history}

The controls in Figure~\ref{fig:sampling_history_controls}(b,c) start from the MOPD student trained with the PG loss at step 100. Two cached batches (seeds 42 and 43) each contain one response per domain, with a 2,048-token prompt cap and a 256-token response cap. All eight responses reach the cap. We sample four positions without replacement per response and weight domains and sampled positions equally; each domain teacher scores its own response.

We compare three Adam states. \emph{Saved} retains both moments and the step counter. \emph{Reset m} zeros the first moment while retaining the second moment and counter. \emph{Fresh} zeros both moments and resets the counter. All gradients are clipped at norm 1, with zero weight decay. A separate zero-gradient control retains the trained moments and is evaluated once. Casting the saved FP32 master weights to BF16 without an optimizer step exactly reproduces the saved model, so this control isolates movement from the retained optimizer state.




\section{Vocabulary-Restricted Losses}
\label{app:losses}

\subsection{Teacher-supported reference loss}
To compare update concentration for different vocabulary supports, we supplement the PG and top-$k$ intersection KL losses with the teacher-top-$k$ reference loss and the full-vocabulary KL loss. Teacher-top-$k$ uses the fixed teacher top-$k$ set $T_k$ and full-vocabulary probabilities, without renormalizing the probabilities of tokens in $T_k$:
\begin{equation}
\ell_{\mathrm{Teacher}\text{-}k}=\sum_{v\in T_k}\left[p_\theta(v)\log\frac{p_\theta(v)}{q_d(v)}-p_\theta(v)+q_d(v)\right].
\label{eq:teacher_topk}
\end{equation}
Teacher-top-$k$ and full KL serve as single-step references; our comparisons use $k=64$ (Teacher-64 in figures). End-to-end training compares the PG loss with the top-$k$ intersection KL loss at $k\in\{16,64\}$.

\subsection{Why we use top-\texorpdfstring{$k$}{k} intersections}
\label{app:intersection_topk}

Our top-$k$ intersection KL loss (Top-$k$) is motivated by an engineering constraint in teacher scoring. Computing a student-top-$k$ loss requires the teacher's scores at the student's top-$k$ token IDs, which depend on the response position. The SGLang teacher interface used in our pipeline does not directly support these position-dependent queries: \texttt{token\_ids\_logprob} accepts one shared list of token IDs for all positions. \citet{zhu2026manyfacesopd} documents this interface mismatch. A workaround queries the union of all student candidate sets at every position, but for a response of length $L$ this returns $L\times|U|$ scores, where $U=\bigcup_{t=1}^{L}S_k(h_t)$, instead of the desired $L\times k$ scores. The union can grow to $\min(|\mathcal V|,Lk)$, substantially increasing memory and data-transfer costs.

We therefore obtain each model's top-$k$ token IDs and scores and match IDs locally at each position, retaining $I_k=S_k\cap T_k$. Both models' scores are available on this intersection without additional teacher queries for student-selected IDs. We renormalize the probabilities of tokens in $I_k$ and apply Equation~\ref{eq:topk}; $k$ is the candidate budget per model, while the supervised support can be smaller.

\citet{li2026rethinkingopd} find that intersection-based supervision recovers nearly all the performance gains of student-top-$k$ OPD in their $k=16$ reasoning experiments. Their result suggests that restricting supervision to shared candidates need not materially reduce OPD performance in that setting, providing empirical motivation for our intersection-based implementation.

\section{Unbiasedness of PG at a Fixed Prefix}
\label{app:pg_unbiasedness}

Fix a domain $d$ and a prefix $h=(x,y_{<t})$. Write $p_\theta(v)=p_\theta(v\mid h)$ and $q_d(v)=q_d(v\mid h)$. Assume a finite vocabulary $\mathcal V$, differentiable student probabilities, and $p_\theta(v),q_d(v)>0$ for every $v\in\mathcal V$. The teacher is frozen, and all gradients below hold the prefix fixed.

\begin{proposition}[Fixed-prefix unbiasedness]
For a token $Y\sim p_\theta(\cdot\mid h)$, the PG surrogate in Equation~\ref{eq:full} satisfies
\begin{equation}
\mathbb E_{Y\sim p_\theta(\cdot\mid h)}
\left[\nabla_\theta\ell_{\mathrm{PG}}(Y)\right]
=\nabla_\theta\ell_{\mathrm{full}},
\label{eq:pg_unbiasedness}
\end{equation}
where each sampled token and the argument of $\sg$ are held constant when differentiating the surrogate.
\end{proposition}

\begin{proof}
For a sampled token $v$, the stop-gradient operation gives
\begin{equation}
\nabla_\theta\ell_{\mathrm{PG}}(v)
=\left(\log p_\theta(v)-\log q_d(v)\right)
\nabla_\theta\log p_\theta(v).
\label{eq:pg_sample_gradient}
\end{equation}
Differentiating the full reverse-KL loss instead gives
\begin{equation}
\begin{aligned}
\nabla_\theta\ell_{\mathrm{full}}
&=\sum_{v\in\mathcal V}\nabla_\theta p_\theta(v)
\left(\log\frac{p_\theta(v)}{q_d(v)}+1\right)\\
&=\sum_{v\in\mathcal V}p_\theta(v)
\log\frac{p_\theta(v)}{q_d(v)}\,
\nabla_\theta\log p_\theta(v)
+\underbrace{\nabla_\theta\sum_{v\in\mathcal V}p_\theta(v)}_{=\,0}\\
&=\mathbb E_{Y\sim p_\theta(\cdot\mid h)}
\left[\nabla_\theta\ell_{\mathrm{PG}}(Y)\right].
\end{aligned}
\label{eq:pg_full_gradient_identity}
\end{equation}
The second line uses $\nabla_\theta p_\theta(v)=p_\theta(v)\nabla_\theta\log p_\theta(v)$ and $\sum_v p_\theta(v)=1$.
\end{proof}

\section{Details of Teacher, Vocabulary, and Optimizer Controls}
\label{app:nine_panel_support}

\subsection{Comparing teachers at a shared student state}
The teacher controls separate differences in supervision from differences in inputs and optimizer history. Figure~\ref{fig:subnetwork_overview}(a) compares cumulative parameter changes under domain-response averaging, with top-64 domain-token and global-token runs included for context; paired points show FP32 and BF16 statistics. Panels (b,c) start from the same MOPD student trained with the PG loss and its saved Adam moments. In the domain-specific comparison, each teacher scores responses from its own domain; in the shared-input comparison, all teachers score the same prefixes. Each of four cached batches contains all six teacher pairs. Small points summarize pairs within a batch, and large points average batches.

To measure the increment attributable to current supervision, Equation~\ref{eq:history_subtracted} subtracts the Adam step obtained with a zero current gradient before computing FP32 update cosines. SGD and reset-$m$ steps are each rescaled to the norm of the saved-state Adam step. Reset-$m$ retains the second moment and step counter.

\subsection{Comparing vocabulary supports at matched prefixes}
Figure~\ref{fig:supervision_overview}(a,b) uses the MOPD students trained with the PG or top-64 intersection KL loss at steps 100 and 500, all trained with domain-response averaging. Each checkpoint supplies four cached batches and 96 prefixes. Within a checkpoint, losses share prefixes and Adam moments; responses are generated separately for each checkpoint. Panel (a) reports batch means. In panel (b), points show batches, lines connect PG batch means, and shading spans the top-64 measurements. The direction comparison includes the PG, top-64 intersection KL, and full-vocabulary KL losses. Panel (c) compares top-16 and top-64 training against PG training at the same task, teacher setting, and training step.

\subsection{Additional optimizer and sampling controls}
\label{app:added_controls}

Figure~\ref{fig:normalization_interventions}(a) uses training logs from 500 steps per averaging rule. Boxes show interquartile ranges of the gradient fractions retained after clipping, whiskers show the 5th--95th percentiles, and labels give the fraction of training steps on which gradients are clipped. Panels (b,c) use the six batches and saved optimizer states described in Appendix~\ref{app:fixed_batch_normalization}. Panel (c) first averages the four domains' probe-loss changes within each batch; error bars then show sample standard deviations computed from the three batches at each checkpoint.

The sampling-count experiment tests how averaging more token samples improves the approximation to the full-vocabulary KL gradient. Figure~\ref{fig:sampling_history_controls}(a) uses three response batches from the MOPD student trained with the PG loss at step 100, with a 4,096-token response cap: 48 responses and 288 prefixes in total. Twelve fixed prefix subsets each receive two independent draws of 1, 16, or 64 actions per prefix. The top-64 intersection KL loss is evaluated once per subset. The mean PG gradient cosine to full KL rises from 0.54 to 0.87 and 0.97 as the sample count increases; top-64 exceeds 0.999. Panels (b,c) use the separate, shorter-response optimizer-history control in Appendix~\ref{app:records}.

\begin{figure}[t]
\centering
\includegraphics[width=\linewidth]{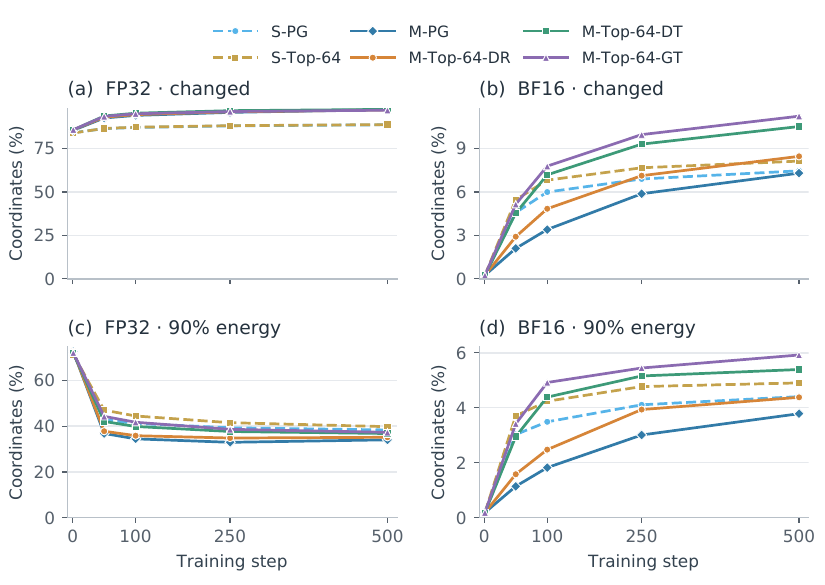}
\caption{\textbf{Cumulative BF16 changes remain much sparser than FP32 changes.}
PG and top-64 trajectories under domain-response averaging in single-teacher and MOPD training, with additional top-64 MOPD runs using domain-token/global-token averaging. Panels (a,b) show nonzero coordinate fractions; panels (c,d) show the fraction of coordinates containing 90\% of squared displacement from initialization. Lines connect steps 1, 50, 100, 250, and 500; dashed lines denote single-teacher training. Figure labels use S/M for single-teacher/MOPD training and PG/Top-64 for the PG/top-64 intersection KL losses.}
\label{fig:geometry}
\end{figure}

\section{Intersection Coverage and Domain-Level Breakdown}
\label{app:completed_results}

\subsection{Probability mass retained by the candidate intersection}
\label{app:actual_intersection_coverage}

To assess how concentrated student predictions support the close gradient agreement in Qwen, we measure the probability retained by the actual student--teacher intersection. The MOPD students trained with the PG or top-64 intersection KL loss, both with domain-response averaging, are evaluated at steps 100 and 500. At each checkpoint, four cached batches provide one response per domain and six sampled positions per response, giving 96 prefixes. Each response is scored by its domain teacher. We record student mass $p_\theta(I_{64}\mid h)$ on the intersection $I_{64}=S_{64}\cap T_{64}$ and its fraction of the student's top-64 mass, $r=p_\theta(I_{64}\mid h)/p_\theta(S_{64}\mid h)$.

Mean intersection mass exceeds 99.9\% at each checkpoint, and the mean retained fraction $r$ exceeds 99.97\%. Therefore, little of the student's top-64 mass is excluded by disagreement with the assigned teacher's candidate set. Among all 384 prefixes, the minimum intersection mass is 96.55\% and every intersection is nonempty. Top-64 therefore retains nearly all student probability throughout the sampled Qwen comparisons.

\subsection{Coverage during top-16 training}
\label{app:i16_online_coverage}

The offline top-64 intersection KL loss audit is complemented by coverage measurements throughout 500 training steps with the top-16 intersection KL loss and domain-response averaging. In 8,000 logged slices of four responses, mean intersection mass is $p_\theta(I_{16})=99.62\%$, mean retained fraction is $r=p_\theta(I_{16})/p_\theta(S_{16})=99.88\%$, and mean intersection size is $|I_{16}|=15.05$. Each slice averages valid tokens within responses and then responses equally. Candidate sets are fixed at rollout, and probabilities are measured during training.

The mean empty-intersection fraction is 0.0027\%. The lowest slice-mean mass is 86.67\% and the largest slice-mean empty-intersection fraction is 1.85\%. These measurements show that high average coverage during training coexists with occasional slices of lower coverage.

\subsection{Domain-level effects of loss averaging}
\label{app:domain_probe_breakdown}

An average probe-loss reduction can conceal an increase on an individual domain. We disaggregate Figure~\ref{fig:normalization_interventions}(c) using the same six batches and twelve update branches per batch (Appendix~\ref{app:fixed_batch_normalization}), with one held-out response per domain. At step 250, saved-state Adam under domain-response, domain-token, and global-token averaging increases instruction-following KL by $0.362/0.293/0.272\times10^{-3}$, respectively; the corresponding four-domain changes are $0.041/-0.002/-0.027\times10^{-3}$. Domain-token and global-token averaging therefore improve the mean probe loss while moving instruction following away from its teacher (Figure~\ref{fig:coverage_domain_controls}).

\begin{figure}[!htbp]
\centering
\includegraphics[width=\linewidth]{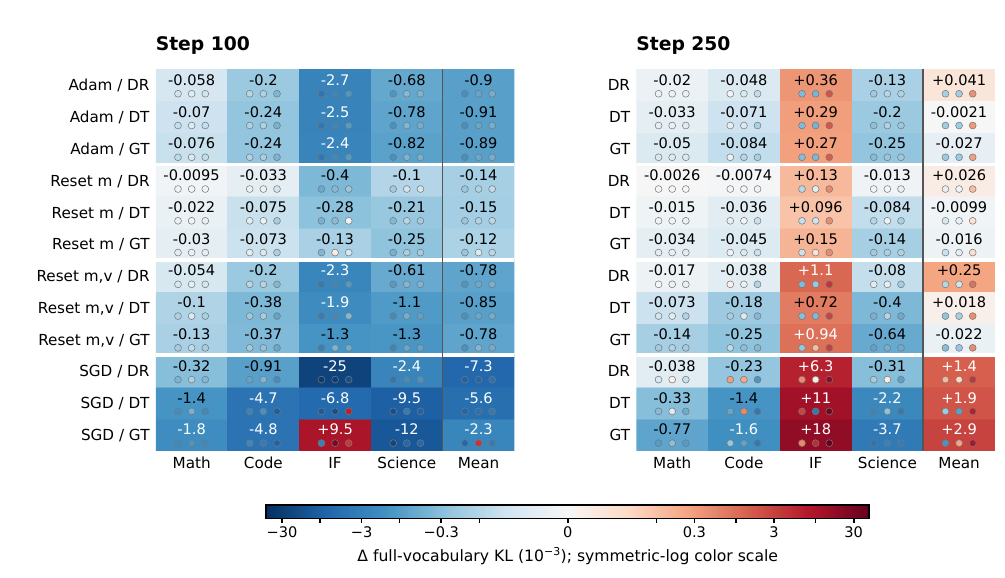}
\caption{\textbf{Mean probe-loss reductions can coexist with domain-level increases.}
Full-vocabulary KL after minus before BF16 writeback, in units of $10^{-3}$. Cells show three-batch means; dots show batches 1042--1044 from left to right. The shared symmetric-log color scale is linear within $\pm0.1$. SGD is matched to each rule's saved-state Adam step norm.}
\label{fig:coverage_domain_controls}
\end{figure}

\section{SmolLM3 Diagnostics and Capability Evaluation}
\label{app:smollm_diagnostics}

\subsection{Models and response coverage}

SmolLM3 provides a complementary setting in which candidate truncation retains less student probability, allowing us to examine gradient approximation beyond the highly concentrated Qwen predictions. We use the SmolLM3-3B student, three RL teachers for mathematics, code, and instruction following, and exported students. The top-64 students use domain-response, domain-token, or global-token averaging with Adam at learning rate $2.5\times10^{-7}$; the PG student uses domain-response averaging with SGD at learning rate $0.0025$.

Diagnostic prompts use Open-MOPD data. We shuffle unique raw messages within each domain and use disjoint coverage, normalization, and probe prompt splits. Generation preserves dataset messages, disables thinking in the chat template, and uses temperature 1, top-$p=1$, a 2,048-token prompt cap, and a 4,096-token response cap. The initialization and the top-64/domain-response student each generate a natural response bank of 384 responses, with equal numbers from each domain. Sampling up to eight uniformly spaced positions per response gives 3,053 prefixes for the initialization and 3,058 for the trained student. Additional students score the trained student's 3,058 prefixes.

Coverage statistics weight domains, responses within domains, and audited positions within responses equally. Mean-mass intervals resample responses independently within each domain 2,000 times using seeds 1042/1043/1044. The initialization and domain-response student use their own natural response banks; the remaining students use the latter bank. On the domain-response student's natural bank, the minimum top-64 intersection mass is $1.7\times10^{-8}$. The domain-token and global-token students each have one empty intersection when scoring the common bank. These low-coverage cases motivate the targeted gradient comparisons below. Figure labels use Initial and DR500 for the initialization and the domain-response student, respectively.

\subsection{Gradient comparisons on typical and difficult prefixes}

Within each domain, we select representative prefixes, prefixes above the domain's 75th entropy percentile, and prefixes whose top-64 intersection mass is below 0.99. The groups can overlap and use up to 32 distinct prompts per domain. The low-coverage code group has 21 eligible prompts, giving the low-coverage group 85 distinct prefixes from all domains combined. Each group is divided into three batches. Smaller candidate budgets, $k=4,8$, are selected using coverage alone.

The BF16 coverage audit sorts a top-128 list and takes its first $k$ entries. Gradient diagnostics recompute top-$k$ in FP32 at the exported BF16 parameter values. Student forward and backward passes use FP32; cached teacher targets use BF16 forward passes followed by FP32 log-softmax. One of the 85 frozen low-coverage prefixes crosses the 0.99 threshold during the FP32 diagnostic and remains in its original group. In FP32, the mean top-64 intersection masses are 99.05\%, 95.68\%, and 93.12\% for the representative, high-entropy, and low-coverage groups, respectively.

Full-vocabulary KL, PG, and top-$k$ intersection KL losses follow Section~\ref{sec:setting}; empty intersections give zero loss. We retain zero- and small-norm cases and report relative gradient error alongside cosine to capture changes in both direction and magnitude. For intersection gradient $g_{I_k}$ and full-vocabulary gradient $g_{\mathrm{full}}$, relative error is $\|g_{I_k}-g_{\mathrm{full}}\|/\|g_{\mathrm{full}}\|$.

\subsection{Supplementary capability observations}

PG training with domain-response averaging and SGD raises MATH-500 accuracy from 50.60\% at initialization to 53.80\%. On AIME 2025, the initialization, the top-64/domain-token Adam student, and the PG/domain-response SGD student each solve four of 30 questions with eight samples per question. Only one solved question is common to all three models.

\end{document}